\documentclass[runningheads]{llncs}

\usepackage{algorithm}
\usepackage{algorithmic}
\usepackage[hidelinks]{hyperref}
\usepackage{xcolor}
\usepackage{comment}

\usepackage{newfloat}
\usepackage{listings}
\floatstyle{ruled}
\newfloat{listing}{tb}{lst}{}
\floatname{listing}{Listing}
\usepackage[utf8]{inputenc}
\usepackage{amsmath} 
\usepackage{fancyvrb}
\usepackage{amsfonts} 
\usepackage{amssymb}
\usepackage{bm}
\usepackage{algorithm}
\usepackage{algorithmic}
\usepackage{comment}
\usepackage{tabularx}
\usepackage{makecell} 
 \usepackage{array,multirow,graphicx}
 \usepackage[misc]{ifsym}

\newcommand{\MAD}{\text{MAD }}

\newcommand\tsource{\ell}
\newcommand\ttarget{m}

\newcommand{\citet}[1]{\cite{#1}}

\makeatletter
\newcommand\footnoteref[1]{\protected@xdef\@thefnmark{\ref{#1}}\@footnotemark}
\makeatother

\everypar=\expandafter{\the\everypar\loosness=-1 }  

\def\X{{\bf X}}
\def\Y{{\bf Y}}

\def\x{{\bf x}}

\def\L{{\bf L}}
\def\R{{\mathbb{R}}}

\def\w{{\bf w}}
\def\GGs{{\boldsymbol\gamma}}
\def\GGt{{\bf \pi}}

\def\GGt{\boldsymbol\pi}
\DeclareMathOperator*{\argmin}{arg\,min}

\title{Match-And-Deform: Time Series Domain Adaptation through Optimal Transport and Temporal Alignment}
\titlerunning{Match-And-Deform: Time Series DA through OT and DTW}

\author {
    François Painblanc \inst{1} \and
     Laetitia Chapel \inst{2} \and
Nicolas Courty \inst{2} \and
Chloé Friguet \inst{2} \and
Charlotte Pelletier \inst{2} \and
Romain Tavenard \Letter \inst{1}
}
\institute{Université Rennes 2, LETG, IRISA, Rennes, France
\and
Université Bretagne Sud, IRISA, UMR CNRS 6074, Vannes, France}

\authorrunning{F. Painblanc, L. Chapel, N. Courty, C. Friguet, C. Pelletier, R. Tavenard}

\tocauthor{F. Painblanc, L. Chapel, N. Courty, C. Friguet, C. Pelletier, R. Tavenard}
\toctitle{Match-And-Deform: Time Series Domain Adaptation through Optimal Transport and Temporal Alignment}

\begin{document}

\maketitle

\begin{abstract}

While large volumes of unlabeled data are usually available, associated labels are often scarce. The unsupervised domain adaptation problem aims at exploiting labels from a source domain to classify data from a related, yet different, target domain.
When time series are at stake, new difficulties arise as temporal shifts may appear in addition to the standard feature distribution shift.
In this paper, we introduce the Match-And-Deform (MAD) approach that aims at finding correspondences between the source and target time series while allowing temporal distortions. The associated optimization problem simultaneously aligns the series thanks to an optimal transport loss and the time stamps through dynamic time warping. When embedded into a deep neural network, MAD helps learning new representations of time series that both align the domains and maximize the discriminative power of the network.
Empirical studies on benchmark datasets and remote sensing data demonstrate that MAD makes meaningful sample-to-sample pairing and time shift estimation, reaching similar or better classification performance than state-of-the-art deep time series domain adaptation strategies.
    
\end{abstract}

\medskip


\section{Introduction}
\label{sec:intro}

A standard assumption in machine learning is that the training and the test data are drawn from the same distribution. When this assumption is not met, trained models often have degraded performances because of their poor generalization ability. 
Domain adaptation (DA) is the branch of machine learning that tackles this generalisation problem when the difference in distribution can be expressed as a shift, allowing for improving task efficiency on a target domain by using all available information from a source domain. 
When dealing with time series data, that are ubiquitous in many real-world situations, being able to learn across time series domains is a challenging task as temporal deformations between domains might occur in addition to the feature distribution shift.

\begin{figure*}[!ht]
    \centering
    \includegraphics[width=\linewidth]{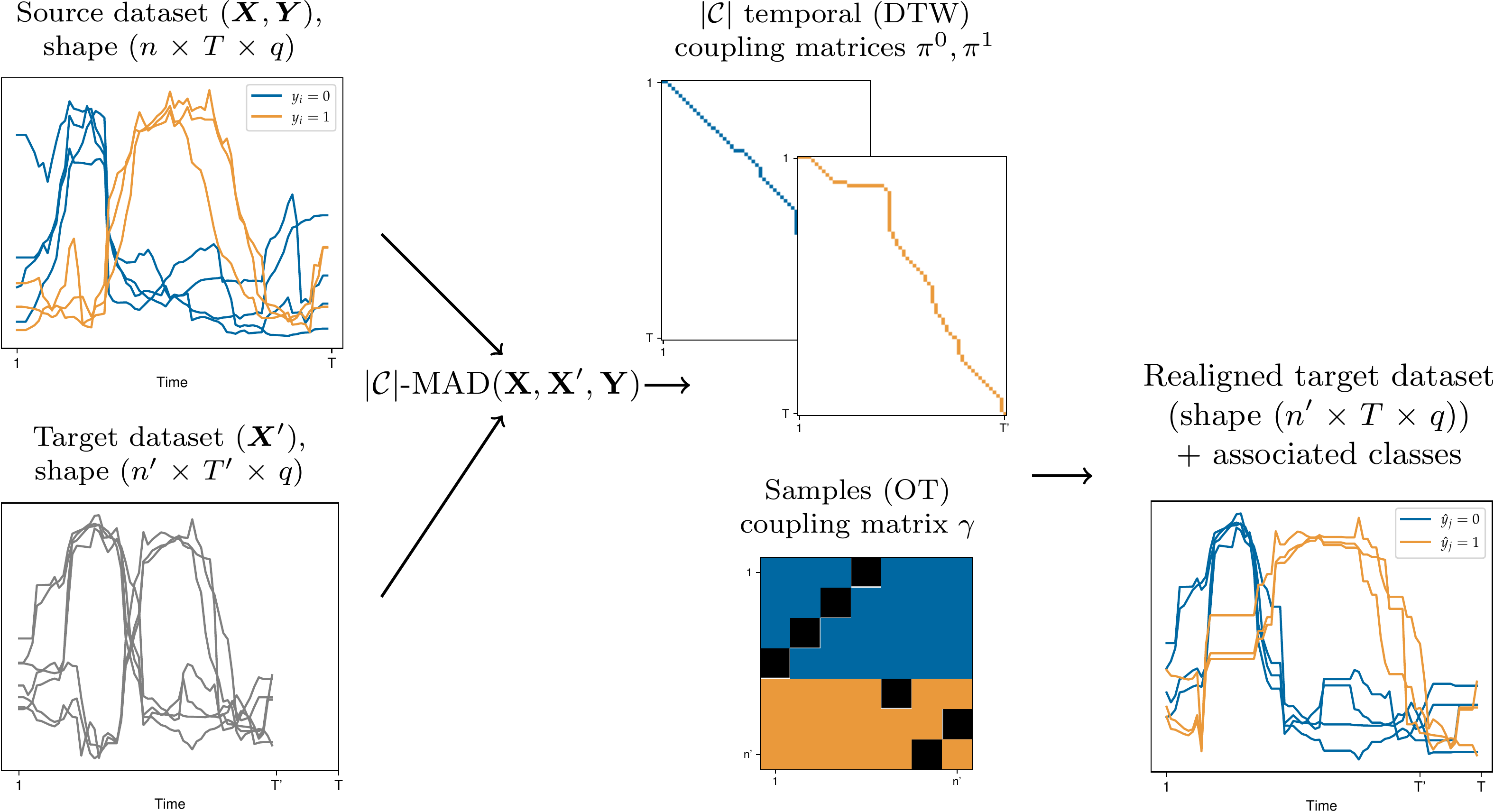}
    \caption{Match-And-Deform ($|\mathcal{C}|$-MAD) takes two time series datasets as inputs: a source (labelled) dataset and a target (unlabelled) dataset. 
    It jointly computes an optimal transport (OT) coupling matrix $\GGs$ and $|\mathcal{C}|$ class-wise dynamic time warping (DTW) paths $\{\GGt^{(c)}\}_{c \in \mathcal C}$.
    The OT cost is derived from the pairwise distances yielded by the DTW paths while the DTW cost is weighted by the OT plan.
    These outputs are then used to improve classification in the target dataset.}
    \label{fig:fig1}
\end{figure*}

Most of the literature in DA focuses on bridging the distribution shift, either by aligning distributions or by finding a common representation space~\cite{SurveyUDA,AdvDA}.
In unsupervised DA frameworks, training is performed on the source domain data relying on this common representation so that effective performance on the target domain can be expected.   
A standard approach consists in using adversarial training to push a deep neural network model into learning domain-invariant intermediate representations~\cite{DANN}.
This approach has been extended to the time series context by using dedicated network architectures.
In VRADA~\cite{VRADA}, variational recurrent neural networks are considered while in the state-of-the-art method CoDATS~\cite{wilson2020codats}, the feature extraction block is made of convolutional layers.
However, these methods operate on pooled features in which the time dimension has vanished and, as such, ignore the impact of potential temporal dynamics across multiple yet related time series. 

Optimal transport (OT) is an efficient tool
for DA, in both unsupervised and semi-supervised learning.
It can be used to evaluate the distribution shift as a non-linear function~\cite{JDOT} or, as done in more recent works, it can help design deep neural network losses that take into account the dissimilarity between the domains~\cite{deepjdot}.
OT-based DA methods have proven useful in many applications but do not encode any temporal coherence that should be kept when dealing with time series data.

The approach presented in this paper, coined Match-And-Deform (MAD), combines OT with dynamic time warping (DTW) to achieve time series matching and timestamp alignment.
This yields a dissimilarity measure that relies on a temporally realigned sample-to-sample matching. 
In other words, \MAD evaluates the feature distribution shift between domains up to a global temporal alignment.
\MAD can then be used as a loss function in a neural network to learn a domain-invariant latent representation.

In the following, the background on DTW and OT are introduced. 
MAD and its use as a loss for DA are then presented and experimentally compared to standard competitors, considering classification tasks on benchmark time series datasets but also on remote sensing time series datasets as an application framework. Finally, we discuss related works and present perspectives. 
\section{Background}
\label{sec:background}

In the sequel, we aim at finding a coupling between samples of two datasets, up to a global temporal shift.
For that purpose, we jointly solve two alignment problems: (i) a temporal realignment (of the timestamps) of the datasets and (ii) a matching between time series (samples). 
Both rely on related linear optimization problems that differ in the nature of the coupling involved. 

In a general way, the optimisation problem for comparing two objects (either time series or distributions) $\textbf{x}$ and $\textbf{x}^\prime$  can be stated as:
\begin{equation}
\label{eq:optimGeneral}
J\Bigl(\textbf{C}(\x,\x^\prime), \Pi\Bigr) = \argmin_{\GGt \in \Pi} \Big\langle \textbf{C}(\x,\x^\prime), \GGt \Big\rangle \, \,,
\end{equation}
in which $\Pi$ is a set of admissible couplings. A coupling will either be a temporal alignment if $\textbf{x}$ and $\textbf{x}^\prime$ are time series or a matching between samples if they are distributions, with appropriate constraint sets, as detailed later in this section.
The solution $J(\cdot,\cdot)$ of the optimization problem is called the optimal coupling matrix.
The cost matrix $\textbf{C}(\x,\x^\prime) = \left\{ d(x^i, x^{\prime j})\right\}_{ij}$ stores distances $d(x^i, x^{\prime j})$ between atomic elements $x^i$ and $x^{\prime j}$, respectively from $\x$ and $\x^\prime$.

\noindent
\textbf{Dynamic time warping (DTW)} is an instance of problem~\eqref{eq:optimGeneral} whose goal is to match two (multivariate) time series $\mathbf{x} \in \R^{T\times q}$, and $\x^\prime \in \R^{T^\prime \times q}$ as:
\begin{equation}
	\text{DTW}(\mathbf{x}, \x^\prime) = J\Big(\textbf{C}(\mathbf{x}, \x^\prime), \mathcal{A}(T, T^\prime)\Big) \, .
	\label{eq:DTW}
\end{equation}
Here, $\textbf{C}(\x,\x^\prime)$ stores Euclidean distances between $q$-dimensional atomic elements and $\mathcal{A}(T, T^\prime)$ is the set of admissible timestamps alignments between series of respective lengths $T$ and $T^\prime$.
An admissible alignment $\GGt \in \mathcal{A}(T, T^\prime)$ is a binary $T\times T'$-matrix that matches timestamps between time series $\mathbf{x}$ and $\x^\prime$.
Its non-zero entries should form a connected path between cells $(1, 1)$ and $(T, T^\prime)$.
This definition allows for efficient computation in quadratic time using dynamic programming~\cite{sakoe1978dynamic}.
Note that DTW can be seen as seeking a temporal transformation (in the form of repeated samples) such that the Euclidean distance between transformed series is minimized.

\noindent
\textbf{Optimal transport (OT)} is another instance of the same general optimization  problem~\eqref{eq:optimGeneral} that defines a distance between two probability measures. 
The discrete probability measures $\X$ and $\X^\prime$ are sets of weighted samples in $\mathbb R^q$: $\{(\mathbf{x}^{i},w^{i})\}_{i=1}^{n}$ with $\sum_{i} w^{i}=1$ and $ \{(\mathbf{x}^{\prime j},w^{\prime j})\}_{j=1}^{n^\prime}$ with $\sum_{j} w^{\prime j}=1$. 
When no prior information is available, weights are set uniformly.
OT defines a distance between $\X$ and $\X^\prime$ by seeking the transport plan $\GGs \in {\Gamma}(\mathbf{w}, \mathbf{w}^\prime)$ that minimises the transport cost:
\begin{equation}\label{eq:wd}
\text{OT}(\X,\X^\prime) = J\Big(\mathbf{C}(\X, \X^\prime), {\Gamma}(\mathbf{w}, \mathbf{w}^\prime)\Big). 
\end{equation}
$\Gamma(\mathbf{w}, \mathbf{w}^\prime)$ is the set of admissible transport plans, \emph{i.e.} the set of linear transports such that all the mass from $\X$ is transported toward all the mass of $\X^\prime$:
$$\Gamma(\mathbf{w}, \mathbf{w}^\prime) =\Big\{\GGs| \GGs \geq \mathbf{0}, \GGs \bm{1}_{n^\prime}=\mathbf{w}, \GGs^\top \bm{1}_{n}=\mathbf{w}^\prime\Big\}$$
with $\bm{1}_n$ a $n$-vector of ones.
The transport plan $\GGs$ is a $n \times n'$-matrix whose elements $\gamma_{ij}$ indicate the amount of mass \emph{transported} from $x^i$ to $x^{\prime j}$. The transport plan is sparse, with at most $n + n' -1$ elements that are non-zeros.  
The most common algorithmic tools to solve the discrete OT problem are borrowed from combinatorial optimisation and linear programming (see~\citet{peyre2019computational} for a thorough review).

\section{Match-And-Deform (MAD)}
\label{sec:mad}

In an unsupervised time series domain adaptation context, let us consider a source dataset $(\X,\Y) \in \mathbb R^{n \times T \times q} \times\mathcal C^n$ and a dataset from a target domain $\X^\prime \in \mathbb R^{n^\prime \times T^\prime \times q}$.  $\X=(x^{i}_1;\ldots;x^{i}_{T})_{i=1}^{n}$ and $\X^\prime=(x^{\prime j}_1;\ldots;x^{\prime j}_{T^\prime})_{j=1}^{n^\prime}$  represent multidimensional time series data ($q\geq 1$) with lengths $T$ and $T^\prime$, respectively. Both domains share the same label space $\mathcal C$ but only source labels $\Y$ are observed. Our goal is to classify target data, using knowledge transferred from source data. 

Despite its widespread use, DTW is limited to finding a temporal alignment between two single time series  $\mathbf{x} \in \R^{T\times q}$, and $\x^\prime \in \R^{T^\prime \times q}$, without any consideration for the alignment of sets of series. 
On the other hand, OT is designed to match datasets regardless of any temporal dimension.
A simple approach to combine both matchings is to use DTW as the inner cost $\mathbf{C}(\X, \X^\prime)$ of an OT problem. However, we argue that this would result in spurious matchings, due to the many degrees of freedom introduced by individual DTW computations (see Experiments section). 
Instead, we introduce a new metric, coined Match-And-Deform (MAD), that jointly optimizes a global DTW alignment and an OT coupling to match two sets of time series, as illustrated in Fig.~\ref{fig:fig1}.
Let us therefore define \MAD as: 
\begin{align}
 \label{eq:mad}
 \text{MAD}(\X, \X^{\prime}) =  &
      \argmin_{\substack{\GGs \in\Gamma(\w,\w^{\prime}) \\ \GGt\in\mathcal{A}(T, T^{\prime})}} \langle \L(\X,\X^{\prime}) \otimes \GGt, \GGs \rangle \nonumber \\
    = &\argmin_{\substack{\GGs \in\Gamma(\w,\w^{\prime}) \\ \GGt\in\mathcal{A}(T, T^{\prime})}} \sum_{i,j} \sum_{\tsource,\ttarget}   d(x^i_\tsource, x^{\prime j}_\ttarget) \pi_{\tsource\ttarget} \gamma_{ij}. 
\end{align}
Here, $\L(\X,\X^{\prime})$ is a 4-dimensional tensor whose elements are $L^{i,j}_{\tsource,\ttarget}=d(x^i_\tsource,x^{\prime j}_\ttarget)$, with $d : \mathbb{R}^q \times \mathbb{R}^q \rightarrow \mathbb{R}^+$ being a distance. 
$\otimes$ is the tensor-matrix multiplication. $\GGt$ is a global DTW alignment between timestamps and $\GGs$ is a transport plan between samples from $\X$ and $\X^{\prime}$. 

By combining metrics that operate between time series and samples in a single optimization problem, MAD gets the best of both worlds: it seeks a transport plan $\GGs$ that matches samples up to a global temporal transformation $\GGt$ (in the form of repeated samples, see the DTW presentation in Background section), hence defining a metric between datasets that is invariant to dataset-wide time shifts.
Note that MAD seeks for a global temporal alignment between datasets, which implies that all series inside a dataset have the same length, otherwise matching timestamps through $\GGt$ would be meaningless.

\label{sec:multi_mad}

The optimization problem in Eq.~\eqref{eq:mad} can be further extended to the case of distinct DTW mappings for each class $c$ in the source data. This results in the following optimization problem, coined $|\mathcal{C}|\text{-MAD}$: 
\begin{equation}
 \label{eq:multi_mad}
 |\mathcal{C}|\text{-MAD}(\X, \X^\prime, \Y) =
      \argmin_{\substack{\GGs \in\Gamma(\w,\w^{\prime}) \\ \forall c, \GGt^{(c)} \in\mathcal{A}(T, T^{\prime})}} \sum_{i,j} \sum_{\tsource,\ttarget} L^{i,j}_{\tsource,\ttarget}  
      \pi^{(y^i)}_{\tsource\ttarget} \gamma_{ij} \, .
\end{equation}
In that case, $|\mathcal{C}|$ DTW alignments are involved, one for each class $c$. $\GGt^{(y^i)}$ denotes the DTW matrix associated to the class $y^i$ of $x^i$.
This more flexible formulation allows adapting to different temporal distortions that might occur across classes.

\noindent \textbf{Properties.}
Let us now study some of the properties of this new similarity measure.
Our first property links Optimal Transport, MAD and $|\mathcal{C}|$-MAD problems.

\begin{property}\label{prop:prop1}
Let $\X \in \R^{n \times T \times q}$ and $\X^\prime \in \R^{n^\prime \times T^\prime \times q}$ be time series datasets.
Let $\text{OT}_\text{DTW}(\X, \X^\prime)$ be the solution of the OT problem with cost $\mathbf{C}(\X, \X^\prime) = \{\text{DTW}(\x,\x^\prime)\}_{\x,\x^\prime \in \X, \X^\prime}$, and let us denote $\text{cost}(\cdot)$ the cost associated to a solution of any optimization problem.
We have: 
\begin{equation*} 
    \text{cost}\Big(\text{OT}_\text{DTW}(\X, \X^\prime)\Big) \leq \text{cost}\Big(|\mathcal{C}|\text{-MAD}(\X, \X^\prime, \Y)\Big) \leq \text{cost}\Big(\text{MAD}(\X, \X^\prime)\Big)
\end{equation*}

\end{property}

Moreover, MAD inherits properties from the OT field.
Typically, the transport plans $\boldsymbol{\gamma}$ resulting from MAD and $|\mathcal{C}|$-MAD are sparse. If $n=n^\prime$ and uniform weights are considered, there always exists a transport plan solution of the MAD (resp. $|\mathcal{C}|$-MAD) problem that is a permutation matrix, as stated in the following property.
\begin{property}
Let $\X$ and $\X^\prime$ be datasets each composed of $n$ time series, and let us assume uniform weights, \emph{i.e.} $\w=\w^\prime=(1/n, \cdots, 1/n)$.
There exists a transport plan solution to the MAD (resp. $|\mathcal{C}|$-MAD) problem that is a one-to-one matching, \emph{i.e.} each sample from $\X$ is matched to exactly one sample in $\X^\prime$ (and conversely).
\end{property} 
Proofs for these properties are provided as Supplementary Material.

\noindent \textbf{Optimization.} \label{sec:optimMAD}
Let us consider the joint optimization problem introduced in Eq.~\eqref{eq:multi_mad}, which involves $|\mathcal{C}|$ finite sets of admissible DTW paths and a continuous space with linear constraints for the OT plan. Extension to solving the problem of Eq.~\eqref{eq:mad} is straightforward.
We perform a Block Coordinate Descent (BCD) to optimize the corresponding loss. 
BCD is a simple iterative algorithm for non-convex optimization problems in which one set of parameters is optimized while the others are held fixed at each step; this algorithm has already been used in similar contexts, e.g. in~\cite{coot_2020}.
In our case, one of the involved maps is optimized with all the other ones fixed, giving $|\mathcal{C}| + 1$ intertwined optimization problems in the form of Eq.~\eqref{eq:optimGeneral}. 
Indeed for a given set of DTW paths $\{\GGt^{(c)}\}_{c=1}^{|\mathcal{C}|}$, the problem in Eq.~\eqref{eq:multi_mad} boils down to a linear OT problem with a cost matrix defined as: 
\begin{equation}
    \label{eq:ot_cost}
    \textbf{C}_\text{OT}(\X, \X^{\prime}, \Y, \{\GGt^{(c)}\}_c) = \left\{\sum_{\tsource,\ttarget} L^{i,j}_{\tsource,\ttarget} 
    \pi^{(y^i)}_{\tsource \ttarget} \right\}_{i, j} .
\end{equation}
Similarly, for a fixed transport plan $\GGs$, solving the problem in Eq.~\eqref{eq:multi_mad} with respect to a given class-specific DTW path $\GGt^{(c)}$ is a DTW problem in which the cost matrix is:
\begin{equation}
    \label{eq:multi_dtw_cost}
    \textbf{C}_\text{DTW}^c(\X, \X^{\prime}, \GGs) = \left\{\sum_{i \text{ s.t. } y^i=c, j} L^{i,j}_{\tsource,\ttarget}
    \gamma_{ij}\right\}_{\tsource, \ttarget} .
\end{equation}
Hence, it leads to optimizing independently $|\mathcal{C}|$ DTW alignment problems.

\begin{algorithm}[t!]
   \caption{$|\mathcal{C}|-$MAD optimization via BCD}
   \label{alg:madbcd}
    \begin{algorithmic}[1]
        \STATE {\bfseries Input:} weighted time series datasets $\X$ and $\X^{\prime}$, initial DTW paths $\{\GGt^{(c)}\}$ 
        
        \REPEAT
            \STATE Compute $\textbf{C}_\text{OT}(\X, \X^{\prime}, Y, \{\GGt^{(c)}\}_c)$ using Eq.~\eqref{eq:ot_cost}
            \STATE $\GGs \gets J\Big(\textbf{C}_\text{OT}(\X, \X^{\prime}, Y, \{\GGt^{(c)}\}_c), \Gamma(\w, \w^{\prime})\Big)$ 
            \FOR{c = 1 .. $|\mathcal{C}|$}
                \STATE Compute $\textbf{C}^c_\text{DTW}(\X, \X^{\prime}, \GGs)$ using Eq.~\eqref{eq:multi_dtw_cost}
                \STATE $\GGt^{(c)} \gets J\Big(\textbf{C}^c_\text{DTW}(\X, \X^{\prime}, \GGs), \mathcal{A}(T, T^{\prime})\Big)$ 
            \ENDFOR
        \UNTIL{\bf convergence}
    \end{algorithmic}
\end{algorithm}
The resulting BCD optimization is presented in Algorithm~\ref{alg:madbcd}.
Note that in each update step, we get an optimal value for the considered map given the other ones fixed, hence the sequence of $|\mathcal{C}|$-MAD losses is both decreasing and lower bounded by zero, thus converging to a local optimum.

\section{Neural domain adaptation with a \MAD loss}

OT has been successfully used as a loss to measure the discrepancy between source and target domain samples embedded into a latent space. Similarly to DeepJDOT~\citet{deepjdot}, our proposal considers a deep unsupervised temporal DA model that relies on \MAD or $|\mathcal{C}|$-\MAD as a regularization loss function.

The deep neural network architecture of DeepJDOT is composed of two parts: (i) an embedding function $g_\Omega$ that maps the inputs into a given latent space, 
and (ii) a classifier $f_\theta$ that maps the latent representation of the samples into a label space shared by both source and target domains.
OT is applied on the output of the embeddings such that $g_\Omega$ yields a discriminant yet domain invariant representation of the data.
When dealing with temporal data, we minimize the following overall loss function  over $\{\GGt^{(c)}\}_c$, $\GGs$, $\Omega$ and $\theta$ :

\begin{equation}
\label{eq:mad-cnn-loss}
\begin{split}
    \mathcal{L}(\X, \Y, \X^{\prime}) =& 
        \overbrace{\frac{1}{n} \sum_{i} \mathcal{L}_{s}(y^{i}, f_\theta(g_\Omega(\x^{i})))}^{\text{(A)}} + 
        \\ 
        & \sum_{i, j} \gamma_{ij} \Big(  \alpha \underbrace{\sum_{\tsource, \ttarget} \pi_{\tsource\ttarget}^{(y^i)} 
        L\left(g_\Omega(\X), g_\Omega(\X^\prime)\right)^{i,j}_{\tsource, \ttarget}
    }_{\text{(B)}} + 
    \beta \underbrace{ \mathcal{L}_{t}(y^{i}, f_\theta(g_\Omega(\x^{\prime j})))}_{\text{(C)}} \Big)
\end{split}
\end{equation}
where $\mathcal{L}_{s}(\cdot, \cdot)$ and $\mathcal{L}_{t}(\cdot, \cdot)$ are cross entropy losses, $\GGs$ (resp. $\{\GGt^{(c)}\}_c$) is the transport plan (resp. the set of DTW paths) yielded by $|\mathcal{C}|\text{-MAD}$. 
The first part (A) of the loss is a classification loss on the source domain; (B) and (C) rely on the OT plan $\gamma_{ij}$:  (C) seeks to align the labels of the source time series with the predicted labels of their matched ($\gamma_{ij}>0$) target time series.
The difference with DeepJDOT lies in the term (B) of Eq.~\eqref{eq:mad-cnn-loss}. This part aims at aligning the latent representations of the time series from the source and target domains that have been matched, the main difference here is that the embeddings of $\x^{i}$ and $\x^{\prime j}$ are additionally temporally realigned thanks to the DTW mapping $\GGt^{(y^i)}$. Finally, $\alpha$ and $\beta$ are hyper-parameters to balance terms (B) and (C) in the global loss. 

This loss can be minimized to zero if (i) the classifier achieves perfect accuracy on source data, (ii) the distribution of features at the output of the embedding for the source and target domains match up to a global temporal alignment per class, and (iii) source labels and target predictions are matched by the MAD transport plan.

\noindent \textbf{Optimization.} 
The loss function described in Eq.~\eqref{eq:mad-cnn-loss} is optimized over two groups of parameters: (i) the neural network parameters $\theta$ and $\Omega$ and (ii)  MAD transport plan $\GGs$ and DTW paths $\{\GGt^{(c)}\}_c$.
Similar to what is done in~\citet{deepjdot}, we use an approximate optimization procedure that relies on stochastic gradients. 
For each mini-batch, a forward pass is performed during which mini-batch-specific $\GGs$ and $\{\GGt^{(c)}\}_c$ maps are estimated using Algorithm~\ref{alg:madbcd}, for fixed network blocks $g_\Omega$ and $f_\theta$.
While the DTW cost matrix presented in Eq.~\eqref{eq:multi_dtw_cost} remains unchanged, the OT cost gets an extra term from the cross-domain label mapping: 
\begin{equation*}
    \Bigg\{\alpha \sum_{\tsource, \ttarget} \pi^{(y^i)}_{\tsource \ttarget} 
    L\left(g_\Omega(\X), g_\Omega(\X^\prime)\right)^{i,j}_{\tsource, \ttarget}
    + \beta \mathcal{L}_{t}\Big(y^i, f_\theta(g_\Omega(\x^{\prime j})\Big) \Bigg\}_{i,j} \, .
\end{equation*}

To initialize the learning process, we use random DTW paths and, when repeating this process for consecutive mini-batches of data, we use the DTW paths from the previous mini-batch as initializers.
Note that even though stochastic optimisation can be expected to converge towards optimal parameters $\Omega$ and $\theta$, it is not the case for the \MAD parameters:  it is known from the OT literature~\citet{genevay2018learning} that the expected value of the OT plan over the mini-batches does not converge to the full OT coupling.
However, authors of~\citet{deepjdot} claim that the resulting non-sparse estimated plan can act as a regularizer to enforce mass sharing between similar samples.

\noindent \textbf{Complexity.}
Assuming time series lengths in both datasets are comparable, and relying on the sparsity of $\GGs$ and $\GGt^{(c)}$, computing the cost matrices $\textbf{C}_\text{OT}$ and $\{\textbf{C}^c_\text{DTW}\}_c$ can be done in $O(b^2 T e + b T^2 e)$ where $b$ is the number of series in a mini-batch and $e$ is the dimension of the embedding $g_\Omega(\cdot)$.
Given a pre-computed cost matrix, each DTW computation is of complexity $O(T^2)$ and similarly, the complexity for the OT solving step is $O(b^3 \log b)$.
The overall complexity for each iteration of Algorithm~\ref{alg:madbcd} is then $O(b^2 (b \log b + Te) + (be + |\mathcal{C}|) T^2)$.
We observe in practice that, in all our experiments, a few iterations were sufficient to reach convergence (characterised by DTW paths not evolving anymore).

\section{Experiments}
\label{sec:xp}

The use of \MAD and $|\mathcal{C}|$-MAD as losses for neural domain adaptation is now assessed considering a real remote sensing dataset, for which there exists a known global temporal shift between the (classes of the) domains due to different weather conditions. We further study its use in a  motion capture context.\footnote{Code, supplementary material and datasets are available at \url{https://github.com/rtavenar/MatchAndDeform}.}
 
\noindent
\textbf{Backbone architecture. } \label{sec:settings}
In order to evaluate the impact of the MAD regularization loss for domain adaptation, we use the exact same deep neural network architecture as in CoDATS \citet{wilson2020codats}. 
The feature extractor $g_\Omega$ is composed of a stack of 3 convolutional layers followed by batch normalization.
The classification head $f_\theta$ then consists of a global average pooling followed by a single fully-connected layer, as shown in Fig.~\ref{fig:madcnn}.
The first layer of the feature extractor has $128$ filters of size $8$, the second layer has $256$ filters of size $5$ and the third layer has $128$ filters of size $3$.
Note that, for CoDATS, the domain adversarial classification head is plugged after the pooling operator.

\begin{figure}[t]
    \centering
    \includegraphics[width=\linewidth]{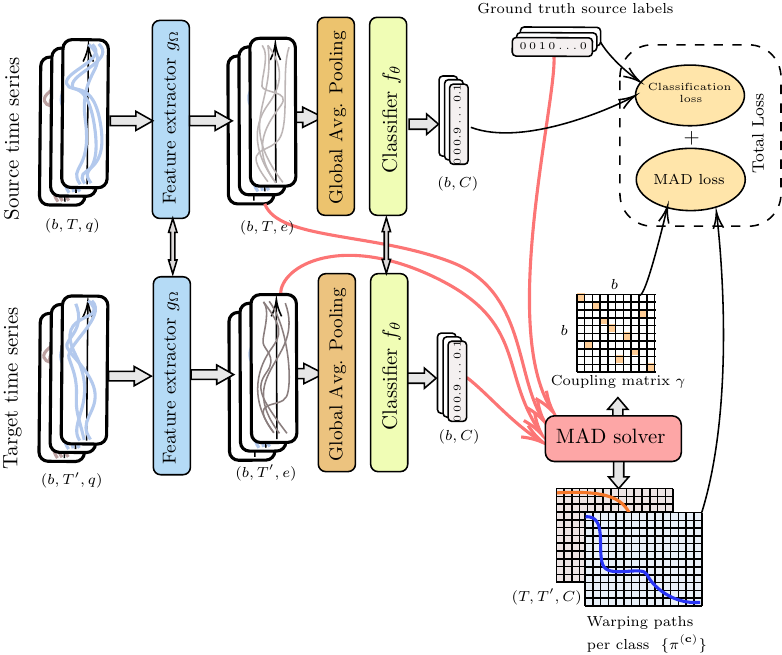}
    \caption{Backbone architecture and schematic view of the loss computation for $|\mathcal{C}|$-MAD.
    \label{fig:madcnn}
    }
\end{figure}

\noindent
\textbf{Hyper-parameters.} For a fair comparison, and for all the experiments, we choose a single learning rate ($0.0001$) and batch size ($b=256$) for both the CoDATS baseline and our approach. The MAD loss in Eq.~\eqref{eq:mad-cnn-loss} introduces two extra hyper-parameters; their values are set to $\alpha=0.01$ and $\beta = 0.01$ for all the experiments. 
In practice, we observe that a large range of values give similar performances (a deeper investigation is provided as Supplementary Material). 

\noindent \textbf{Experimental setup. }
Our proposed method \MAD and its variant $|\mathcal{C}|$-\MAD are compared to CoDATS, which extends the domain adversarial neural network (DANN, \citet{DANN}) framework to time series data and has been shown to reach state-of-the-art performance.
As in \citet{wilson2020codats}, the proportion of each class in the target domain is assumed to be known during the learning phase. This corresponds to the CoDATS-WS variant.
As an OT for DA baseline, we furthermore compare to DeepJDOT-DTW, considering DeepJDOT \cite{deepjdot} adapted for time series: individual DTW between pairs of series are considered as transport cost matrix $C(\boldsymbol{X}, \boldsymbol{X'}) = \{\text{DTW}(\boldsymbol{x}, \boldsymbol{x'})\}_{\boldsymbol{x}, \boldsymbol{x'}\in \boldsymbol{X}, \boldsymbol{X'}}$.

In MAD and DeepJDOT-DTW, the weights $\textbf{w}$ from the source mini-batches are set such that they reflect the proportion of the classes in the target domain; uniform weights $\textbf{w}^\prime$ are used for target mini-batches. 
Note that integrating this extra information could be avoided by using unbalanced optimal transport \citet{fatras2021unbalanced}, which is left for future work.
For the sake of fairness, all competing methods are trained and tested on the same data splits with the same backbone architecture. 

Classification accuracy from the same model trained on source data to predict the target labels (hereafter denoted \textit{No adaptation}) is reported for each domain pair. This gives a proxy of the overall difficulty of the adaptation problem for this pair.
The \textit{Target only} baseline is trained and evaluated on the target domain, and can be seen as an upper bound estimation of the classification accuracy.
For each domain, the data is split into train/validation/test sets with 64\%~/~16\%~/~20\% of the data in each split, respectively. Note that the validation set is used only for CoDATS-WS for early stopping purpose.
For each experiment, the average over all pairs of domains is provided together with the averaged standard deviations.

\noindent
\textbf{Remote Sensing data.}
\label{sec:expe:RS}
We first focus on an application field for which the temporal shift between domains has already been documented. In Earth observation, and especially for land cover mapping applications, the differences in weather, soil conditions or farmer practices between study sites are known to induce temporal shifts in the data that should be dealt with, as discussed in~\cite{nyborg2021timematch} for example.
Moreover, these temporal shifts can be global to the study sites or class specific (see Fig.~\ref{fig:NDVI}), which is the setup for which \MAD and $|\mathcal{C}|$-\MAD are designed.

\begin{figure}[t!]
    \centering
    \includegraphics[width=\linewidth]{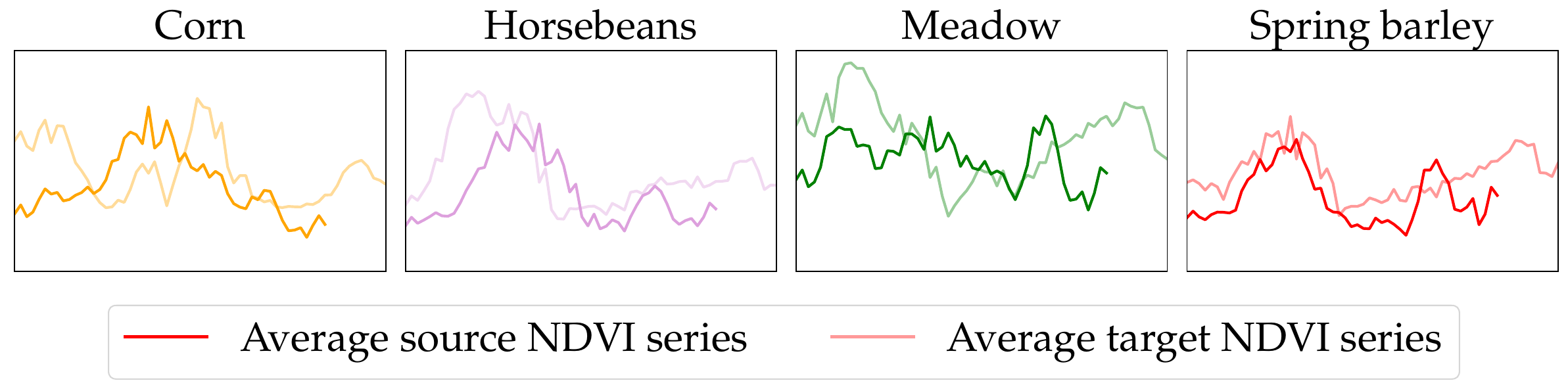}
    \caption{Per-class median-filtered average Normal Difference Vegetation Index (NDVI) profiles, showing the growth of 4 crop types (miniTimeMatch dataset, 1 domain pair (DK1 $\rightarrow$ FR1)). Crops develop similarly in the 2 considered regions, but the patterns are temporally shifted, with a class-specific shift.}
    \label{fig:NDVI}
\end{figure}

TimeMatch dataset~\cite{nyborg2021timematch} is a crop-type mapping dataset covering four different geographical areas in Austria, Denmark, and France. It contains time series of multi-spectral measurements  from satellite imagery at each geo-location. Labels are available at a parcel level: hence we  consider a modified version of TimeMatch, referred to as miniTimeMatch in the following, in which measurements are averaged over agricultural parcels. 
The aim is to recognize parcels' crop type, such as corn or wheat. Domains correspond to different geographical areas, with their own characteristics.
Moreover, we remove data from classes that are not observed in all domains, or with less than $200$ agricultural fields in at least one domain. {We also remove the "unknown" class that gathers crops from many types.}
Finally, we sample $n_\text{min}$ series per class in each domain, where $n_\text{min}$ is the minimum frequency of the class across all domains.
This results in a dataset of 
{$28,858$} time series and {$8$} classes per domain, each being described by $10$ features per timestamp. 
The same domain pairs as in \cite{nyborg2021timematch} are considered for the evaluation.

Observations from the miniTimeMatch dataset may be noisy, with some measurements corresponding to cloudy observations. Moreover, the time series of the different domains have different temporal sampling due to the filtering of highly cloudy images (see Figure~\ref{fig:NDVI}). 
Then, we introduce a new domain adaptation problem, considering another land crop mapping dataset called TarnBrittany, in which observations from cloudy images are removed.
It is built from Sentinel-2\footnote{\url{https://theia-ide.org/}} time series for the year 2018, and is composed of two domains: (Brittany) in the mid-west of France on tile T30UVU and (Tarn) in the southwest of France on tile T31TDJ.
We collect 72 images for Tarn and 58 images for the cloudier Brittany.
Images have been orthorectified and corrected from atmospheric effects by using the MAJA processing chain~\cite{lonjou2016maccs} to yield Level-2A surface reflectance images. 
As for miniTimeMatch, we average the reflectance values over the agricultural fields to obtain multivariate time series of 10 spectral bands.
Labels are obtained from the French land parcel identification system (IGN RPG\footnote{\url{https://geoservices.ign.fr/rpg}}) for the year 2018.
We keep only the five main crops present in both areas. A total of 5,743 series is sampled in each area with similar class proportions in source and target domains. 
The two domain pairs are considered for the evaluation.

\noindent
\textbf{Motion Capture (MoCap) data.} \label{sec:expe:MoCap}
We further evaluate \MAD and $|\mathcal{C}|$-\MAD on a MoCap setup: Human Activity Recognition (HAR) dataset~\citet{Dua2019}, in which no temporal shift is documented. As such, we do not expect \MAD and $|\mathcal{C}|$-\MAD to behave better than state-of-the-art algorithms.
The dataset contains time series (128 timestamps, 9 features, 6 classes) of movements recorded thanks to mobile sensors such as accelerometer or gyroscope; the aim is to recognize human activities such as walking or sitting.
Domains correspond to different persons with their own characteristics. The same 10 domain pairs as in~\citet{wilson2020codats} are considered for the evaluation.
Details on the dataset are provided as Supplementary Material. 

\noindent
\textbf{Results} Tab.~\ref{tab:tab_acc_mocap} shows that, on HAR dataset,  \MAD and $|\mathcal{C}|$-\MAD reach similar performance to CoDATS-WS. In more details, \MAD and $|\mathcal{C}|$-\MAD outperform  CoDATS-WS in five out of ten situations, achieving a slightly better average performance.  
Tab.~\ref{tab:acc_timematch} provides classification accuracy for remote sensing datasets, on which \MAD and $|\mathcal{C}|$-\MAD are expected to work better thanks to their ability to handle (per class) time shifts.
The discrepancies between performances of models trained only on source data (\textit{No adaptation}) and those only trained on target data (\textit{Target only}) also point out that all domain pairs constitute a more challenging adaptation problem compared to HAR dataset.
One can notice that \MAD and $|\mathcal{C}|$-\MAD outperform CoDATS-WS in 5 out of the 7 problems, sometimes with an important improvement (see DK1 $\rightarrow$ FR1 for  example), even in the presence of noise (that is present in the miniTimeMatch dataset).  
This illustrates the fact that \MAD and $|\mathcal{C}|$-\MAD are of prime interest when global or class-specific temporal deformations occur between domains (see Fig.~\ref{fig:NDVI}).
For each set of datasets, a one-sided Wilcoxon signed-rank test is performed to assert if the results of \MAD and $|\mathcal{C}|$-\MAD are statistically greater than those of CoDATS-WS and deepJDOT-DTW. 
Significant results ($\alpha=0.05$) are obtained for \MAD and $|\mathcal{C}|$-\MAD over CoDATS-WS ($p=0.005$ and $p=0.01$, respectively), and over DeepJDOT-DTW ($p= 3.10^{-6}$ and $p=6.10^{-6}$, respectively), confirming the superiority of our approaches. There is no significant difference ($p=0.18$) between CODATS-WS and deepJDOT-DTW performances. 

\begin{table*}[t!]
  \centering
  \begin{tabular}{|r|c|c|c|c|c|c|}
    \hline
    Problem             & No adapt.      &  CoDATS-WS                    & DeepJDOT-DTW    & \MAD{} & $|\mathcal{C}|$-MAD & Target only \\ \hline \hline
    ~2 $\rightarrow$ 11 & $83.3 \pm 0.7$ &  $81.3  \pm 4.4$              & \boldsymbol{$99.5 \pm 0.7$} & $98.4 \pm 1.3$ & $99.0 \pm 1.5$ & $100 \pm 0.0$ \\
   ~7 $\rightarrow$ 13  & $89.9 \pm 3.6$ &  $94.4  \pm 2.6$              & $88.4 \pm 0.7$ & \boldsymbol{$100.0 \pm 0.0$} & \boldsymbol{$100 \pm 0.0$} & $100 \pm 0.0$ \\
   12 $\rightarrow$ 16  & $41.9 \pm 0.0$ &  \boldsymbol{$64.0  \pm 0.6$} & $56.8 \pm 3.8$ & \boldsymbol{$64.0 \pm 0.6$} & \boldsymbol{$64.0 \pm 0.6$} & $100 \pm 0.0$ \\
   12 $\rightarrow$ 18  & $90.0 \pm 1.7$ &  \boldsymbol{$100 \pm 0.0$}   & $94.5 \pm 0.0$ & $99.5 \pm 0.6$ & $99.5 \pm 0.7$ & $100 \pm 0.0$ \\
   ~9 $\rightarrow$ 18  & $31.1 \pm 1.7$ &  \boldsymbol{$78.1  \pm 4.9$} & $51.6 \pm 10.0$ & $71.7 \pm 0.6$ & $71.2 \pm 1.1$ & $100 \pm 0.0$ \\
   14 $\rightarrow$ 19  & $62.0 \pm 4.3$ &  \boldsymbol{$99.5 \pm 0.7$}  & $64.4 \pm 1.3$ & $83.3 \pm 2.3$ & $84.3 \pm 2.4$ & $100 \pm 0.0$ \\
   18 $\rightarrow$ 23  & $89.3 \pm 5.0$ &  $89.8  \pm 0.6$              & $97.3 \pm 0.0$ & \boldsymbol{$98.2 \pm 0.6$} & $97.8 \pm 0.6$ & $100 \pm 0.0$ \\
   ~6 $\rightarrow$ 23  & $52.9 \pm 2.3$ &  $95.1  \pm 3.5$              & $60.9 \pm 2.7$ & \boldsymbol{$97.8 \pm 0.6$} & \boldsymbol{$97.8 \pm 0.6$} & $100 \pm 0.0$ \\
   ~7 $\rightarrow$ 24  & $94.4 \pm 2.7$ &   $99.6 \pm 0.6$               & $98.3 \pm 1.2$ & \boldsymbol{$100.0 \pm 0.0$} & \boldsymbol{$100 \pm 0.0$} & $100 \pm 0.0$ \\
   17 $\rightarrow$ 25  & $57.3 \pm 5.5$ &   {$95.5  \pm 4.6$} & $73.6 \pm 3.2$ & \boldsymbol{$96.3 \pm 2.0$} & $95.5 \pm 1.2$ & $100 \pm 0.0$ \\ \hline
   Average              & $69.2 \pm 2.8$ &  $89.7 \pm 2.3$               & $78.5 \pm 2.4$ & \boldsymbol{$90.9 \pm 0.9$} & \boldsymbol{$90.9 \pm 0.9$}  & $100 \pm 0.0$ \\ \hline

\end{tabular}
  \caption{HAR dataset: classification performance (\% accuracy, avg over 3 runs $\pm$ std).  
  \label{tab:tab_acc_mocap}}

\end{table*}

\begin{table*}[t!]
    \centering
    \hspace*{-0.5cm}
   \begin{tabular}{|r|c|c|c|c|c|c|}
    \hline
    Problem 		      & No adapt.  & CoDATS-WS      & DeepJDOT-DTW   & \MAD           & $|\mathcal{C}|$-MAD         & Target only \\ 
    \hline \hline
    Tarn $\rightarrow$ Brittany & $88.5 \pm 3.7$ & $96.0 \pm 1.6$ & $89.4 \pm 2.2$ & $98.8 \pm 0.3$ & \boldsymbol{$98.9 \pm 0.4$} & $99.7 \pm 0.1$ \\
    Brittany $\rightarrow$ Tarn & $48.9 \pm 0.6$ & \boldsymbol{$93.6 \pm 0.1$} & $47.6 \pm 1.4$ & $92.0 \pm 0.2$ & $90.6 \pm 1.1$ & $98.6 \pm 0.2$ \\ \hline
    Average                     & $68.7 \pm 1.7$ & $94.8 \pm 0.9$ & $68.5 \pm 1.8$ & \boldsymbol{$95.4 \pm 0.2$} & $94.7 \pm 0.8$ & $99.2 \pm 0.1$ \\ 
    \hline \hline
    DK1 $\rightarrow$ FR1 & $69.2 \pm 1.3$ & $74.8 \pm 1.5$ & $ 79.0 \pm 1.6	$ & \boldsymbol{$88.4 \pm 0.4 $} & $ 88.3 \pm 0.9$ & $95.8 \pm 0.9 $ \\
    DK1 $\rightarrow$ FR2 & $62.2 \pm 3.5$ & $\boldsymbol{87.0 \pm 3.4}$ & $76.6\pm 2.5 	$ & $ 82.5 \pm 1.1 $ & $ 81.0 \pm 1.1$ & $94.2 \pm 1.7$ \\
    DK1 $\rightarrow$ AT1 & $73.9 \pm 0.2$ & $71.6 \pm 15.4$ & $ 78.6 \pm 0.6	$ & \boldsymbol{$ 93.1 \pm 1.2 $} & $92.3 \pm 2.2$ & $96.7 \pm 0.7$ \\
    FR1 $\rightarrow$ DK1 & $61.9 \pm 5.2$ & $78.0 \pm 10.7$ & $71.3 \pm 2.9$& \boldsymbol{$ 88.2 \pm 0.3 $} &\boldsymbol{$ 88.2 \pm 0.5 $}&  $96.2 \pm 0.3$ \\
    FR1 $\rightarrow$ FR2 & $78.8 \pm 0.9$ & $82.1 \pm 8.2$ & $ 77.1 \pm 1.3	$ & \boldsymbol{$ 90.5 \pm 0.2$} & {$ 89.6 \pm 0.4$} & $94.2 \pm 1.7$ \\ \hline
    \hline 
    Average & $ 69.2 \pm 2.2$ & $78.7 \pm 7.8$ & $76.5 \pm 1.8$ & \boldsymbol{$88.5 \pm 0.6$} &	$87.9 \pm 1.0$ & $ 95.4 \pm 1.1$ \\
    \hline 
    \end{tabular}
    \caption{Remote sensing datasets: classification performance (\% accuracy, avg over 3 runs $\pm$ std).
    \label{tab:acc_timematch}}
\end{table*}

\noindent \textbf{$|\mathcal{C}|$-\MAD as a trade-off between \MAD and DeepJDOT-DTW.}
$|\mathcal{C}|$-\MAD\ can be seen as an intermediate configuration between MAD (in which a single global DTW alignment is performed) and DeepJDOT-DTW (which allows individual alignments between each pair of series).
The good performance of $|\mathcal{C}|$-\MAD compared to DeepJDOT-DTW show that it successfully manages the specific intra-class global alignments.
By computing pair-to-pair DTW alignments, DeepJDOT-DTW allows for too many degrees of freedom that lead to spurious matchings. On the contrary, ($|\mathcal{C}|$-)\MAD acts as a regularizer (see Property~\ref{prop:prop1}), allowing to meaningfully constrain the alignments for better performance.

\noindent \textbf{Latent space visualization.}
To further investigate the internal properties of the models, we use Multi-Dimensional Scaling (MDS) to visualize the latent space of the compared models.
To do so, we focus on the models learned for the ``Tarn $\rightarrow$ Brittany'' adaptation problem and visualize $g_\Omega(\X)$ and $g_\Omega(\X^\prime)$ jointly for a model trained on source data only, as well as for CoDATS-WS and $|\mathcal{C}|$-\MAD that aim at learning a domain-invariant representation.
While Euclidean distance is used to feed MDS for both baselines, we use the distances resulting from the DTW paths for $|\mathcal{C}|$-\MAD (as stored in $\textbf{C}_\text{OT}$), in order to account for the temporal realignment on which our method relies.
The resulting visualizations presented in Fig.~\ref{fig:latent_codats} show a significant shift in distributions from target to source domain, when no DA strategy is employed.
CoDATS-WS also fails to align domains in the presence of time shifts, whereas $|\mathcal{C}|$-\MAD successfully clusters series by class.

\begin{figure}[!t]
    \centering
    \includegraphics[width=.7\linewidth]{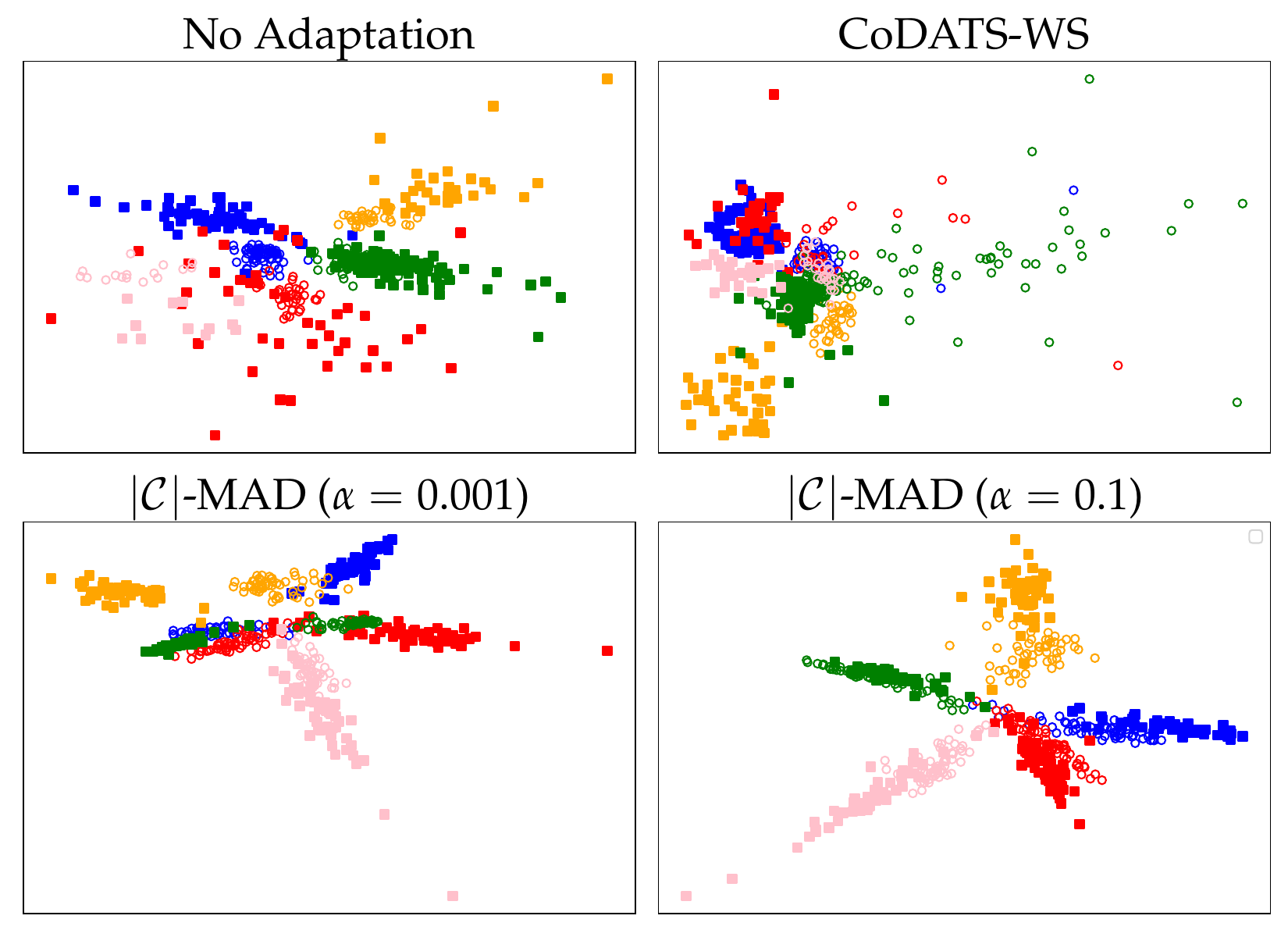}
    \caption{Multi-Dimensional Scaling (MDS) representation of the latent space for several models on the Tarn (empty circles)~$\rightarrow$~Brittany (squares) adaptation problem: each colour represents a different class.
    }
    \label{fig:latent_codats}
\end{figure}

\section{Related work}
\label{sec:rel}

Several approaches proposed in the literature share some common ground with MAD, either because they aim at aligning complex objects to perform domain adaptation tasks or because they deal with the problem of aligning (sets of) time series.

Using OT for domain adaptation has been initially proposed in OTDA~\cite{OTDA}, which aims at aligning the source and target domains using OT. It relies on the hypothesis that the distribution shift between the two domains is an affine transformation and uses the transport plan to evaluate its coefficients. Classification is then performed on the realigned samples. 
DeepJDOT~\cite{deepjdot}, on which the neural domain adaptation with a \MAD loss introduced in this paper relies, can be seen as an improvement over OTDA. 
CO-OT~\cite{coot_2020} jointly optimizes two OT plans $\GGs_s$ and $\GGs_f$ between samples (rows) and features (columns) of datasets $\X$ and $\X^\prime$, resulting in a loss of the form $\langle \L(\X,\X^\prime) \otimes \GGs_s, \GGs_f \rangle$. It has been benchmarked in a heterogeneous domain adaptation context but does not allow enforcing temporal coherence that should be met when time series are at stake.
To the best of our knowledge, no OT-based approach has yet tackled the challenge of unsupervised DA for time series classification.

Combining OT and (soft-)DTW has also been proposed in~\citet{Janati2020} where the goal is to align brain imaging signals using soft-DTW with an OT-based ground cost.
In this setting however, optimization is not performed jointly on $\GGs$ and $\GGt$ but, rather, individual transport problems are solved independently for each pair of timestamps, hence not enforcing global spatial consistency. \citet{Janati2022} propose an extension with the aim to average spatio-temporal signals.
However, none of these works consider the DA scenario.

In the same line of work, \citet{cohen2021aligning} use a Gromov-Wasserstein-like similarity measure to compare time series that may not lie in the same ambient space (\textit{e.g.} that do not share the same features). 
They solve a problem of the form $\langle \L(\X,\X^\prime) \otimes \GGt, \GGt \rangle$ (in which $\GGt$ is a coupling matrix that meets the DTW alignment constraints), coined Gromov-dynamic time warping. 
Optimization also relies on a BCD, however its convergence is not guaranteed.
Moreover, it does not allow taking into account the temporal distortion that occurs between \textit{sets} of time series.

\section{Conclusion and perspectives}\label{sec:conclu}
In this paper, we introduce Match-And-Deform (MAD) that combines optimal transport and dynamic time warping to match time series across domains given global temporal transformations. We embed \MAD as a regularization loss in a neural domain adaptation setting and evaluate its performance in different settings: on MoCap datasets, \MAD gives similar performance to state-of-the-art DA time series classification methods; when a temporal deformation is at stake (such as in a remote sensing scenario), we show that \MAD reaches better performance, thanks to its ability to capture temporal shifts. 

The $|\mathcal{C}|$-MAD variant presented in this paper relies on source domain classes to form groups of temporally coherent series, yet an unsupervised strategy to form groups based on time series content alone is a worthy track for future works.
Moreover, inspired by~\cite{vincent-cuaz2022semirelaxed}, time series weights could be learned instead of set \emph{a priori}, in order to decrease the sensitivity to outliers.
Finally, the use of \MAD in other tasks, such as cross-domain missing data imputation~\cite{muzellec2020missing}, can also be considered.

\subsubsection*{Acknowledgements.}
François Painblanc and Romain Tavenard are partially funded through project MATS ANR-18-CE23-0006. 
Nicolas Courty is partially funded through
project OTTOPIA ANR-20-CHIA-0030. Laetitia Chapel is partially funded through project MULTISCALE ANR-18-CE23-0022.

\bibliographystyle{splncs04}  
\bibliography{biblio.bib}
\end{document}


\maketitle

This document contains part of the Supplementary Material for the ``Match-And-Deform'' paper, namely additional details that could not fit in the paper as well as proofs for mathematical properties related to MAD.

The Supplementary Material for this paper also includes code to reproduce the experiments presented in the paper as well as miniTimeMatch and TarnBZH datasets.

\section{Proofs}

\begin{property}\label{prop:prop1}
Let $\X \in \R^{n \times T \times q}$ and $\X^\prime \in \R^{n^\prime \times T^\prime \times q}$ be time series datasets.
Let $\text{OT}_\text{DTW}(\X, \X^\prime)$ be the solution of the OT problem with cost $\mathbf{C}(\X, \X^\prime) = \{\text{DTW}(\x,\x^\prime)\}_{\x,\x^\prime \in \X, \X^\prime}$, and let us denote $cost(\cdot)$ the cost associated to a solution of any optimization problem.
We have:
\begin{eqnarray}
    \text{cost}(\text{OT}_\text{DTW}(\X, \X^\prime)) &\leq& \text{cost}(|\mathcal{C}|\text{-MAD}(\X, \X^\prime, \Y)) \nonumber \\
    \text{cost}(|\mathcal{C}|\text{-MAD}(\X, \X^\prime, \Y)) &\leq& \text{cost}(\text{MAD}(\X, \X^\prime)) \nonumber
\end{eqnarray}
\end{property}

\begin{proof}
Let us start with the first inequality.
To do so, let us denote by $\GGs^\star,(\GGt^{(1) \star}, \cdots, \GGt^{(c) \star})$ the optimal coupling matrices resulting from $|\mathcal{C}|\text{-MAD}(\X, \X^\prime,\Y)$.
The cost associated to $|\mathcal{C}|\text{-MAD}(\X, \X^\prime, \Y)$ is:
$$
    \sum_{i,j} \gamma^\star_{ij} \underbrace{\sum_{\tsource, \ttarget} L^{i,j}_{\tsource, \ttarget} \pi^{(y_i) \star}}_{A(\x^i, \x^{\prime j})}
$$
By definition of the DTW, we have $A(\x^i, \x^{\prime j}) \geq \text{DTW}(\x^i, \x^{\prime j})$, which gives:
$$
    \text{cost}(|\mathcal{C}|\text{-MAD}(\X, \X^\prime, \Y)) \geq \sum_{i,j} \gamma^\star_{ij} \text{DTW}(\x^i, \x^{\prime j})
$$
Once again, since $\GGs^\star$ is a valid transportation plan, the right-hand term above is greater than the cost of the OT problem that relies on DTW as its inner cost, which proves our first inequality.

Let us now prove that $|\mathcal{C}|$-MAD upper bounds MAD.
To do so, let us now denote $\GGs^\star,\GGt^\star$ the optimal coupling matrices resulting from $\text{MAD}(\X, \X^\prime)$.
Let us observe that $\GGs^\star, \underbrace{(\GGt^\star, \cdots, \GGt^\star)}_{|\mathcal{C}| \text{ times}}$ is a valid solution candidate for the $|\mathcal{C}|$-MAD problem.
As a consequence, the $|\mathcal{C}|$-MAD cost is lower than that of MAD, which concludes the proof.
\end{proof}

To assert empirically\footnote{Notebooks for empirical proofs of properties 1 and 2 can be found with the code} Property~1, we compare the costs of $\text{OT}_\text{DTW}$, $|\mathcal{C}|\text{-MAD}$ and $\text{-MAD}$. We draw two random datasets of size $(50 \times 20 \times 2)$ with 5 classes. We compute the respective costs of the three methods over these two datasets and compared over $2,000$ repetitions. Figure~\ref{fig:proof1} shows that the cost of $\text{OT}_\text{DTW}$ is always the smallest  while the cost of \MAD{} is always the greatest among the 3 methods..

\begin{figure}
    \centering
    \includegraphics[width=0.5\textwidth]{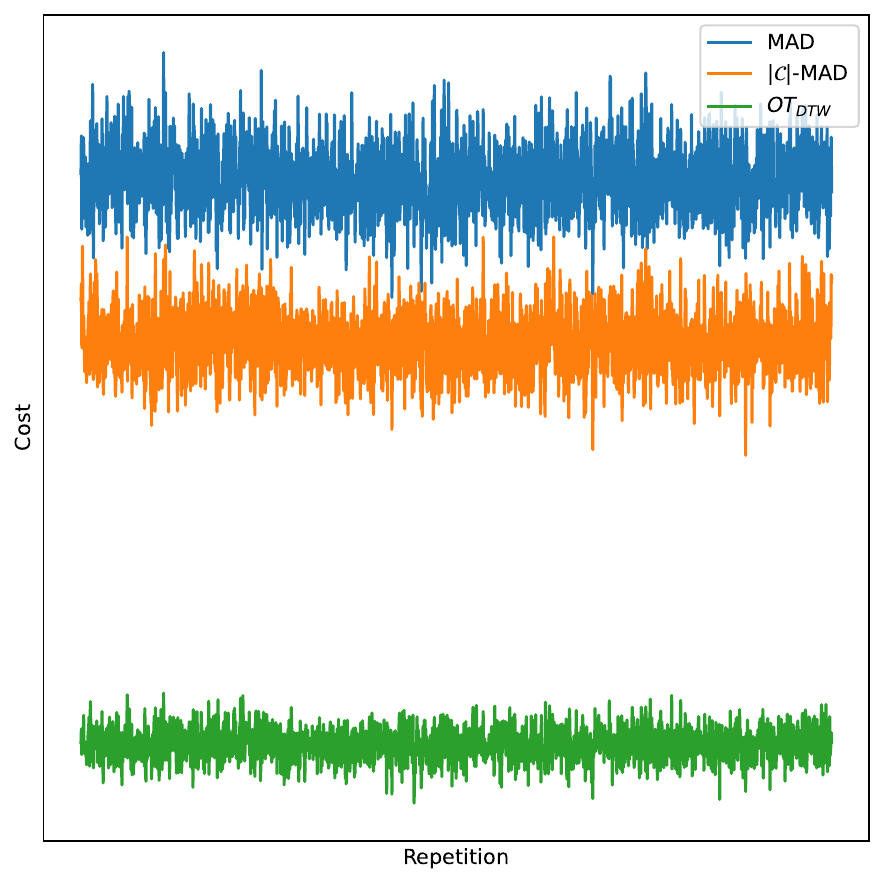}
    \caption{Costs of $\text{OT}_\text{DTW}$, $|\mathcal{C}|\text{-MAD}$ and $\text{MAD}$ over $2,000$ random datasets of size $(50 \times 20 \times 2)$.}
    \label{fig:proof1}
\end{figure}

For our second proof, we will rely on the following Lemma:

\begin{lemma}
\label{lemma}
The MAD optimization problem can be re-written:
$$
\text{MAD}(\X, \X^{\prime}) =
      \argmin_{\GGt\in\mathcal{A}(T, T^{\prime})}
      \underbrace{\argmin_{\GGs \in\Gamma(\w,\w^{\prime})} \langle \L(\X,\X^{\prime}) \otimes \GGt, \GGs \rangle}_{\mathbf{\text{OT}}_\text{MAD}(\X, \X^\prime, \GGt)}
$$
\end{lemma}

\begin{proof}
This new formulation straight-forwardly derives from the observation that the set $\mathcal{A}(T, T^\prime)$ of admissible alignments involved in MAD is finite, hence the joint optimization can be seen as seeking the minimum $\text{OT}_\text{MAD}$ cost along enumeration of the admissible alignments in $\mathcal{A}(T, T^\prime).$
\end{proof}
Note that a similar Lemma could be obtained for $|\mathcal{C}|$-MAD using the exact same argument. We omit it here for the sake of brevity.

\begin{property}
Let $\X$ and $\X^\prime$ be datasets each composed of $n$ time series, and let us assume that uniform weights are employed, \emph{i.e.} $\w=\w^\prime=(1/n, \cdots, 1/n)$.
There exists a transportation plan solution to the MAD (resp. $|\mathcal{C}|$-MAD) problem that is a one-to-one matching, \emph{i.e.} each sample from $\X$ is matched to exactly one sample in $\X^\prime$ (and conversely).
\end{property}

\begin{proof}
We will prove this property for MAD, knowing that the exact same reasoning can be employed for $|\mathcal{C}|$-MAD.

Let us use the re-writing of the MAD optimization problem given by Lemma~\ref{lemma}.
For any $\GGt$, the problem $\mathbf{\text{OT}}_\text{MAD}(\X, \X^\prime, \GGt)$ is an Optimal Transport problem with cost $\L(\X,\X^{\prime}) \otimes \GGt$.
One can hence apply Proposition 2.1 from~\citet{peyre2019computational} to prove that each of these OT problems admits a solution that is a permutation matrix.
We can deduce that there exists a transportation plan solution to the MAD problem that is a permutation matrix, \emph{i.e.} a transportation plan in which each sample from $\X$ is matched to exactly one sample in $\X^\prime$ (and conversely).
\end{proof}

To assert empirically Property~2, we compare the number of matching given by the MAD transport plan to the number of series on simulated datasets. We draw two datasets $1,998$ number of times from size $(2 \times 20 \times 2)$ to size $(2000 \times 20 \times 2)$, increasing gradually the number of series in the datasets. Whatever the data size, the number of matching indeed corresponds exactly to the number of series in the datasets.

\section{Experimental details}

\noindent \textbf{Experimental tools.} All experiments concerning DeepJDOT-DTW, MAD, $|\mathcal{C}|$-MAD  were ran using CPU only on $3,000$ iterations on two Xeon intel 20 E5-2687W CPUs. CoDATS and baselines "\textit{No adaptation}" and "\textit{Target only}" where trained over $30,000$ iterations using a GPU NVIDIA TITAN RTX.

\noindent \textbf{Dataset statistics.}
Table~\ref{tab:recap_data} reports a brief description of the datasets used throughout the paper and in this Supplementary Material.
\begin{table*}[ht]
  \centering
  \begin{tabular}{|l|l|l|l|l|l|}
    \hline
    Dataset & domains & $\#$ samples & TS length & dim & $\#$ classes  \\
    & & per domain & & & \\\hline \hline
    HAR      & 15 domains, as in \citet{wilson2020codats} & $[288;409]$ & 128 & 9 & 6 \\
    miniTimeMatch       & DK1, FR1, FR2, AT1 & 21,648& 52; 62; 39; 58 & 10 & 7 \\
    TarnBZH & Tarn, Brittany & 5,743 & 72; 58 & 10 &5\\\hline
    
  \end{tabular}
  \caption{Description of the domain adaptation datasets. When figures are given into brackets, it corresponds to the minimum and maximum values over the dataset.\label{tab:recap_data}}
\end{table*}

\noindent \textbf{Influence of parameters $\alpha$ and $\beta$.}
\label{sec:expe:hp}

In all experiments presented in the paper, a fixed value is used for parameters $\alpha$ and $\beta$.
The sensitivity of MAD to these hyper-parameters is now investigated and illustrated on the TarnBZH domain adaptation problems. 
First, $\alpha$ is set to $10^{-1}$ and accuracy for a range of values for $\beta$ are reported on Figure~\ref{fig:hp_effect} (right). It shows that accuracy remains stable across $\beta$ values, for both adaptation problems.
It may appear like being non-critical to the overall performance.

Then, $\beta$ is set to $10^{-2}$ and accuracy for a range of values for $\alpha$ are reported on Figure~\ref{fig:hp_effect} (left).
The impact of $\alpha$ is a bit stronger and, as expected, setting $\alpha$ to 0 (and hence not taking temporal information into account) leads to much degraded performance.
For a large range of $\alpha$ values, however, accuracy remains stable. 
Finally, one can see that an $\alpha$ value too large diminishes the overall performance by not putting enough emphasis on the classification performance.

\begin{figure*}[!h]
    \centering
    \includegraphics[width=.45\linewidth]{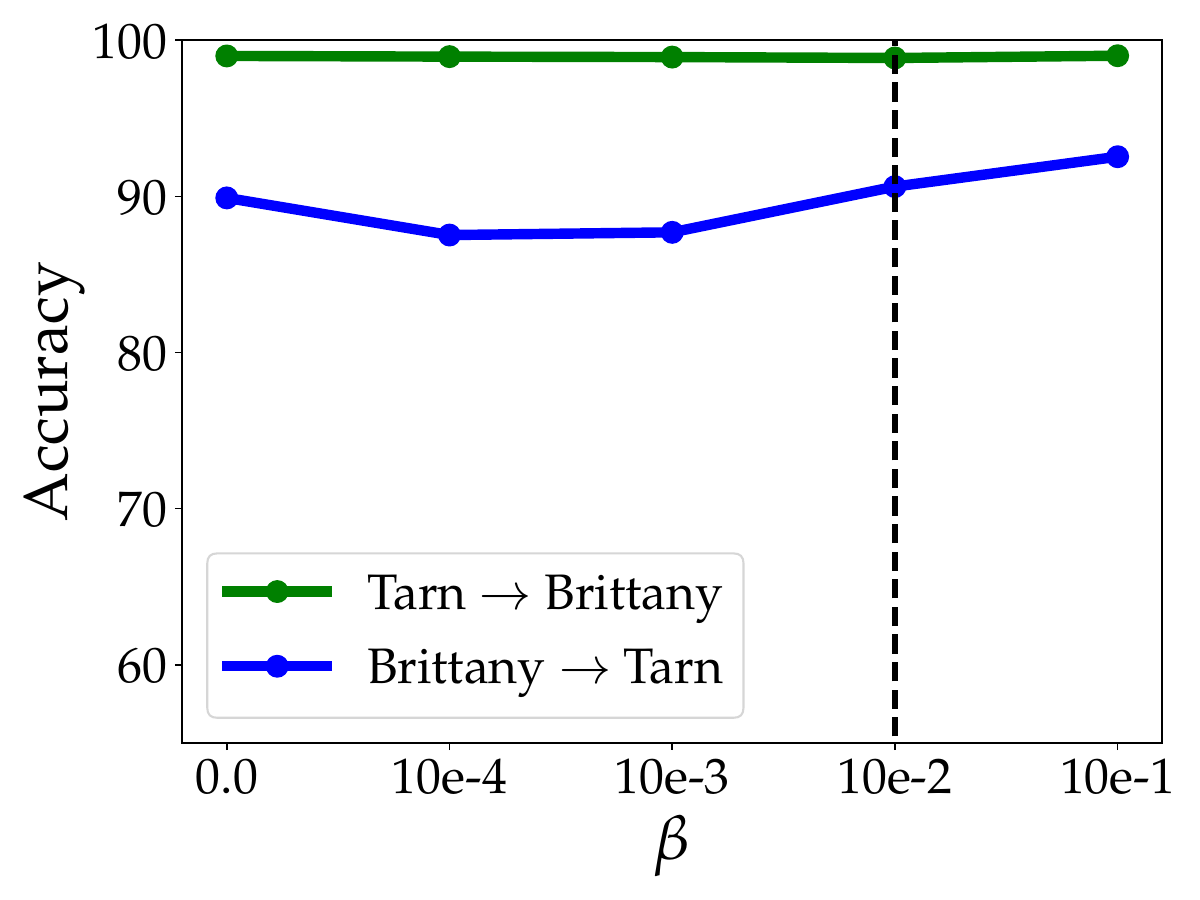}
    \includegraphics[width=.45\linewidth]{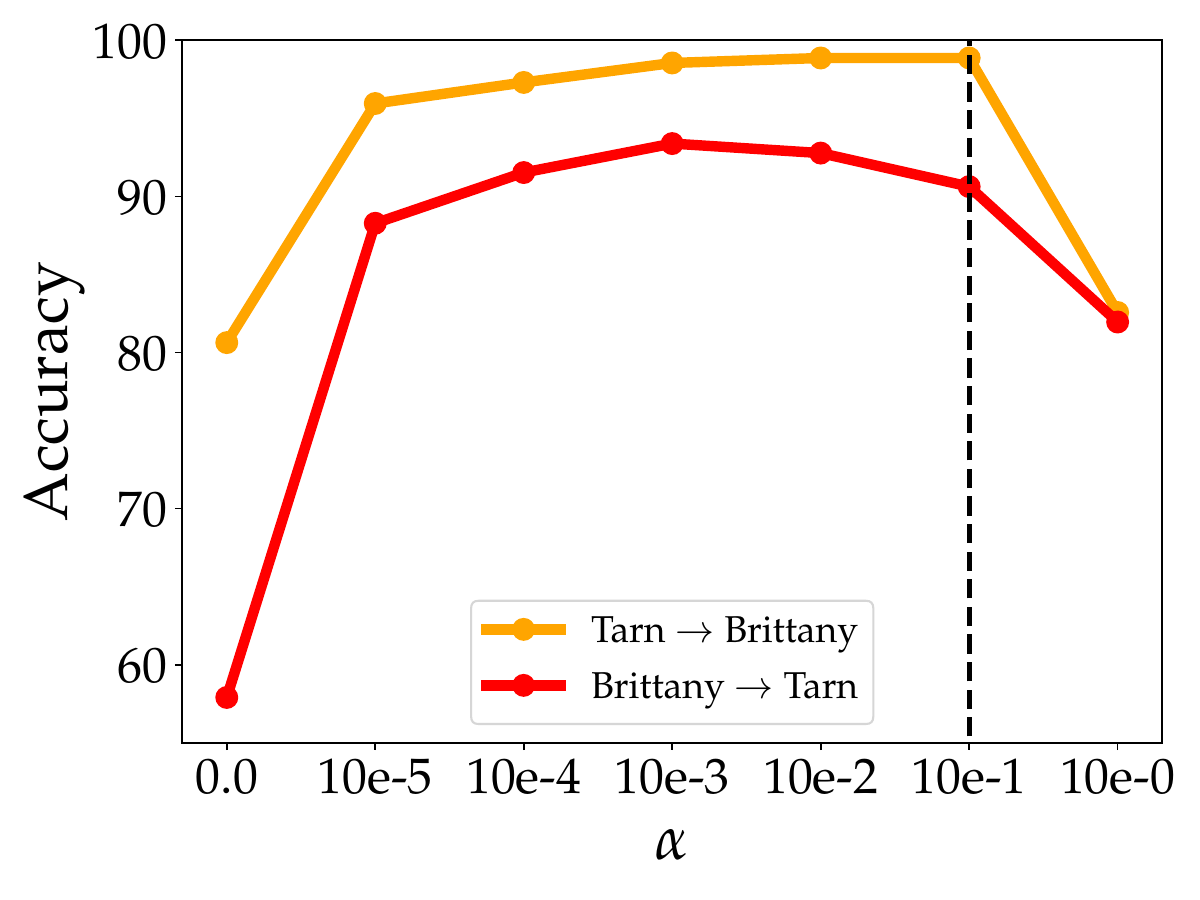}
    \caption{Classification performance as a function of (left) $\beta$ (with $\alpha=10^{-1}$) and (right) $\alpha$ (with $\beta=10^{-2})$.
    Vertical dashed lines correspond to the values used in all other experiments in this paper.
    }
    \label{fig:hp_effect}
\end{figure*}  

\bibliographystyle{splncs04}

\bibliography{biblio.bib}